\documentclass{article}
\usepackage{times}
\usepackage{latexsym}
\usepackage{amsmath,amsfonts,amssymb, amsthm, bbm}
\usepackage{comment}
\usepackage[utf8]{inputenc} 
\usepackage[T1]{fontenc}    
\usepackage{hyperref}       
\usepackage{url}            
\usepackage{booktabs}       
\usepackage{amsfonts}       
\usepackage{nicefrac}       
\usepackage{geometry}
\usepackage{microtype}      

\newcommand{\st}{\operatorname{s.\!t.}}

\newcommand{\Hs}{\mathcal{H}}
\newcommand{\Fs}{\mathcal{F}}
\newcommand{\Ms}{\mathcal{M}}

\newcommand{\charfun}{\mathbbm{1}}

\newcommand{\R}{\mathbb{R}}
\DeclareMathOperator*{\E}{\mathrm{E}}
\DeclareMathOperator*{\Var}{\mathrm{Var}}

\newtheorem{thm}{Theorem}
\newtheorem{df}{Definition}
\newtheorem{lem}{Lemma}

\title{Ordering as privileged information}
\author{Tom Vacek \\
University of Minnesota, Minneapolis, MN \\
vacek@cs.umn.edu}

\begin{document}
\maketitle

\begin{abstract}
We propose to accelerate the rate of convergence of the pattern recognition task by directly minimizing the variance diameters of certain hypothesis spaces, which are critical quantities in fast-convergence results.
We show that the variance diameters can be controlled
by dividing hypothesis spaces into metric balls based on a
new order metric.
This order metric can be minimized as an ordinal regression problem, leading to
a LUPI application where we take the privileged information as some desired ordering,
and construct a faster-converging hypothesis space by empirically restricting some
larger hypothesis space according to that ordering.  We give a risk analysis of the approach.  We discuss the difficulties with model selection and give an 
innovative technique for selecting multiple model parameters.
Finally, we provide some data experiments.
\end{abstract}

\section{Introduction}

Learning using Privileged Information (first proposed by Vapnik et.\ al.\ \cite{VapnikLUPI})
seeks to bring in privileged information to assist the learner.  This
information is privileged because the learner may use it choose a hypothesis, but
the privileged information will be unavailable making decisions based on the hypothesis.

This paper proposes a LUPI method that directly minimizes the variance 
diameters of the hypothesis spaces under consideration, an essential quantity
in fast-convergence literature \cite{MammenTsybakov, tsybakov2004}.
This approach applies to discriminant-based hypotheses spaces where 
predictions are derived from thresholding the
the discriminant value, which should be totally orderable.
We show that the discriminant ordering defines equivalence classes for the
elements of the space, and these classes are directly related
the the variance diameters we seek to control.
If we could restrict the hypothesis space to just the good equivalence classes,
we would reduce the variance diameters and improve the speed of convergence.

Selecting a good equivalence class requires some external definition of
desirable order.  This is privileged information.
This raises a natural question:\emph{What is a desirable order?}
If ordering information were provided by an oracle according
to some true distribution, then \emph{any} ordering would 
provide desirable variance diameters.  
However, since an empirical ordering is the best we can hope for,
a good ordering is one which has favorable convergence properties
in ordinal regression.\footnote{We would want the ordering to 
correspond to a good 
hypothesis for the pattern recognition problem, though this is
irrelevant to the \emph{rate} of convergence.}
Low ordinal loss is sufficient, while more 
general characterizations may be possible.

From another perspective, orderings according to conditional probability $P(Y|X)$ have
great appeal, as models that provide good estimates of conditional probability
allow broader application than ones that do not.
Moreover, confidence seems to be a sliver of common ground between human and machine 
learning.
Thus, we consider orderings which
seem to bear some relationship to conditional probability, though quite loose.
While a total ordering is best for controlling variance diameters,
two independent orderings for each class can be shown to be very nearly as good.
This arrangement allows us to expand the scope of tasks where 
useful privileged information is available.

\section{What is order?}
We can't hope to provide any kind of overview of the field of order statistics.  
Web search made this a lucrative and popular field.
Nevertheless, we believe our starting point is novel:
\begin{df}
\textbf{Order} is any property of a set of real numbers that is invariant under any
invertible increasing transformation.
\label{def:order}
\end{df}
Ordering naturally defines equivalence classes on hypotheses spaces.
Suppose some
distribution $P$ generates feature vectors $X \in \R^d$,
and suppose
$h_1$ and $h_2$ are hypotheses in a space $\Hs: \R^d \rightarrow \R$.
If an increasing function $m$ exists so that for all $X \sim P$,
$m(h_1(X)) = h_2(X)$,\footnote{We might relax this to holding on a set of full measure.} then $h_1$ and $h_2$ are in the same equivalence class.

The thread which connects order to the pattern recognition task
is the growth function, which
measures the number of possible labelings of a set of points of size $n$
by a 0/1-hypothesis set $\Hs^{0/1}$.  We assume that $\Hs^{0/1}$ is defined by 
characteristic functions of real-valued functions, as in $\Hs^{0/1} =
\{\xi(h(X) - t): h \in \Hs: \R^d \rightarrow \R, t \in \R\}$.
We observe that if $\hat \Hs$ contains a restricted number
of order equivalence classes, then the growth function
of $\Hs$ is also restricted.  We propose that the relationship can be made precise
by the machinery of variance-based risk bounds. 

As a thought experiment, consider the variance diameter of $h_1$ and $h_2$
when combined with an appropriate 0/1 loss function, 
and restricted to one class.  More formally, suppose now that 
$P$ generates feature vectors and labels $Y \in \pm 1$ jointly.
\begin{df} 
Zero-one loss is:
\begin{equation}
l^{01}(\hat y, y) = \left \{ 
\begin{array}{ll} 
1 & \hat y>0, y \leq 0\\
1 & \hat y \leq 0, y > 0 \\
0 & \mbox{otherwise}
\end{array}
\right .
\label{def:01Loss}
\end{equation}\footnote{
As a matter of boilerplate, we use loss functions
as are commonly defined in machine learning literature: a loss
 function takes two arguments: a prediction and a label; however,
we may omit those where obvious to avoid clutter.  Loss functions
are uniquely identified by a superscript, unless intended to be
taken generally.  The expectation of the loss $\E_{X,Y \sim P} [l(h(X),Y)]$ is
depicted as $L(h)$, and the empirical risk on a finite set of size $n$ is
depicted as $L_n(h)$.  There are numerous more measurability and
existence assumptions that we will not cover here.
}
\end{df}
Then for any $h_1$ and $h_2$ in the same order equivalence class
for for any $t_1$ and $t_2$, it is easy to show that
\begin{align}
&\E_{P\restriction_{Y=1}} |l^{01}(h_1(X) -t_1, Y) - l^{01}(h_2(X)-t_2,Y) | \\
\leq & |\E_{P\restriction_{Y=1}} l^{01}(h_1(X) -t_1, Y) - l^{01}(h_2(X)-t_2,Y) |
\label{equivSingleClassVariance}
\end{align}
Except for the restriction to a single class, this is just the relationship for variance diameters\footnote{
Fast converging bounds require the following: Denote by $h' \in \Hs$ the minimum of the
true risk over $\Hs$.  The following must hold uniformly for all $h \in \Hs$:
\begin{align}
\begin{split}
& \Var[l^{01}(h(X),Y) - l^{01}(h'(X),Y)] \\
\leq &  \E[l^{01}(h(X),Y) - l^{01}(h'(X),Y)] 
\end{split}
\label{withinClassVariance}
\end{align}
Our presentation is slightly different because the variance can be upper bounded by the 
expectation of the absolute value and we have not specialized to $h'$.
Note that many presentations of fast convergence can lead an in-attentive reader
to believe that $h'$ must be the 
Bayes rule.  This is true if one is using the Mammen-Tsybakov noise conditions
to establish the desired variance relationship, but not necessary if the 
relationship can be established another way, as we do here.
}
that is needed
for fast rates.

There are two tasks required to extend the thought experiment
to a real learning formulation.
First, the relationship needs to apply in both classes simultaneously.  
The complicating factor to just adding the two per-class relationships is that the 
absolute-value arguments in the two respective RHS's could 
have opposite signs and cancel out.
This can happen when the decision threshold falls in very different places relative to the ordering.
We fix this by requiring that the loss in the two classes be balanced as a constraint on
valid solutions.  In effect, this is a constraint on the location of the 
decision boundary in the equivalence class definition,
preventing the situations where there is cancellation.

More significantly, there is no access to the equivalence classes based only 
on empirical information.  This is the majority of the analysis in the
remainder of the paper.  In short, we relax the notion of the equivalence class
to metric balls, and then we bound bound the deviation of
empirical balls from their true diameter.

\subsection{Ordering metric}
We define a metric on orderings so that two 
hypotheses are in the same equivalence class if their metric distance 
is 0.
The metric, when shown to have favorable properties,
allows us to create balls of restricted variance diameter 
based on a finite sample using
ordinary empirical risk minimization.

\begin{df}
Let $\mathbf{\Ms}$ be the set of all increasing continuous functions.
\end{df}
\begin{df}
Let $P$ generate vectors $X \in \R^d$,
and consider any two functions $h_1,h_2: \R^d \rightarrow R$.
Then the order distance between $h_1$ and $h_2$ is
\[D(h_1, h_2) = \sup_{t \in \R} \inf_{m \in \Ms} \E_P[l^{01}(m \circ h_1(X)-t, h_2(X)-t)] \]
\end{df}
The metric axioms (up to equivalence class elements) are
not hard to check; the invertability of $m$ is indispensable here.

While many definitions would
satisfy the metric axioms, we chose this definition for two reasons.
First, it is a direct extension of the result we saw for equivalence classes.
If $D(h_1(X), h_2(X)) \leq d$, then (\ref{equivSingleClassVariance}) holds
with $d$ added to the right-hand side.
Second, the fact that 0/1 loss is an underlying component allows us
to borrow a great deal from the standard results in machine learning.

Since we have no access to true probabilities in a statistical learning setting, to proceed
we have to derive a way to get access to $D$.  We accomplish this by
extending $D$ to measure the order distance of a function $h$ to some \emph{ground truth}
ordering instead of another function.  The key observation is that functions within $D_0$ of the ground truth ordering
are within $2D_0$ of each other by the triangle inequality---an empirical version
of the equivalence classes we set out to find.
We will redefine this extension of $D$ as $L^{iso}$ for clarity:
\begin{df}
Suppose $P$ jointly generates feature vectors $X$
and real-valued labels $Y$.  Let $\Hs: \R^d \rightarrow \R$ be some
hypothesis space.
\[L^{iso}(h) = \sup_{t \in \R} \inf_{m \in \Ms} \E_P[l^{01}(m \circ h_1(X)-t, Y-t)] \]
\end{df}

The task at hand is to
analyze risk bounds for $L^{iso}$
that allows us, with high probability,
to identify functions where $L^{iso}(h) \leq D_0$.
based on a finite sample.
The key observation is that the infimum ($m \in \Ms$) and supremum ($t \in \R$)
in the definition
of $L^{iso}$ can be handled by a uniform bound on the deviations of 
a loss function over the joint space $\Hs \times \Ms \times \R$.
We define that loss function to be $l^{thr}$:
\begin{df}
\[l^{thr}(\hat y, y, t) = l^{01}(\hat y-t, y-t) \]
\end{df}
We will 
show that if $\Hs$ has bounded level VC-dimension,\footnote{
Level VC-Dimension is an extension of VC-dimension to 
real-valued functions.  It is the VC dimension of $\{h-t: h \in \Hs, t \in \R\}$.}
then the loss class for $l^{thr}$ over $\Hs \times \Ms \times \R$ also has bounded
VC dimension, so
we can derive uniform risk bounds for
empirical risk minimization.

We observe that level VC-dimension already satisfies
an order invariance.  That is, the inclusion of 
$\Ms$ doesn't have any affect on the estimated growth function.
\begin{lem}
Let $\Ms \Hs = \{m \circ h: h \in \Hs, m \in \Ms\}$.
Then the level VC dimension of $\Ms \Hs$ is the same
as for $\Hs$.
\end{lem}
The proof is a simple shattering argument.
Obviously, $\Ms \Hs \supseteq \Hs$, so the VC-dimension is 
not less.  It is not more because
any rule $m \circ h - t > 0$ can be 
replicated by $h - t' > 0$ for an appropriately chosen $t'$.

These two observations give the following theorem:

\begin{thm}
\label{thm:ordinalRegressionBound}
Suppose $\Hs$ has level VC-dimension $V$.
There exists a universal constant $C$ such that 
uniformly for all $h \in \Hs$ the following holds
with probability at least $1-\delta$:
\begin{align*}
&L^{iso}(h)
\leq \sup_{t \in \R} \inf_{m \in \Ms} L_n^{thr}(m \circ h,t) + C\sqrt{\frac{V \log(n) + \log{\frac{1}{\delta}}}{n}}
\end{align*}
\end{thm}
The bound is a straightforward application of uniform risk bounds, provided that the VC dimension of the loss class can be found.  This is a 
straightforward shattering argument.  Suppose $\Hs$ has level VC dimension
$V$.  Suppose there exists a set $\{(x_i, y_i)\}_{i=1}^{2V+1}$ and some $t_0$ such that $\{l^{thr}(h(x_i), y_i, t_0)\}_{h \in \Hs}$ can attain any labeling.  Considering the two sets $\{i: y_i > t_0\}$ and $\{i: y_i \leq t_0\}$, one of these has size at least $V+1$ by the pigeonhole principle.
Call the set $S$.
Then $\{h(x_i)-t_0 > 0\}_{h \in \Hs, i \in S}$ is shattered, which contradicts our assumption about the level VC dimension of $\Hs$.  Hence, the VC-dimension is at most $2V$.
Extending the argument to allow separate orderings in each class is similar;
there are now four pigeonholes because we have to consider a separate
threshold for each ordering, so $4V+1$ points will ensure a contradiction.

\section{Putting the pieces together}
The ordinal regression bound is the substantial 
piece of the puzzle, but to make 
use of it have to enforce class balance.
Suppose that the user provides some parameter
$w$ as a weight to favor or discourage loss in one 
class over another:

\begin{df}
\label{lossBalance}
Loss balance, given a specified parameter $w > 0$, is measured by
\[l^B(\hat y, y; w) = \left \{ \begin{array}{ll} 
\max(w,1) & \hat y \leq 0, y > 0 \\
-\max(1,\frac{1}{w}) & \hat y > 0 , y \leq 0
\end{array} \right .
\]
\end{df}
The VC dimension of this loss class can be analyzed 
in terms of the VC dimension for the underlying 
hypothesis space just like in the ordinal regression
application, giving a similar bound.
With this in hand, we are ready for the main theorem:
\begin{thm}
\label{thm:riskBound}
Suppose some predefined ordering is provided.  
Consider the subset of $\Hs$ that satisfies 
the predefined ordering up to loss $d$ and 
assume that the user provides a loss balance parameter
that is empirically satisfied:
\[\hat \Hs_d = \{ h \in \Hs: L_n^{iso}(h) \leq d, L^B_n(h;w) = 0 \}. \]
Assume that this set is not empty.
Let $\hat h_n$ be the empirical minimizer (of $L_n^{01}$) and
$\hat h'$ be the true minimizer (of $L^{01}$)
over $\hat \Hs$.
Then there exists a constant $C$ such that the following 
holds uniformly for all $h \in \hat \Hs$ and for all $\phi \leq 1$
with probability at least $1 - \delta_1 - \delta_2 - \delta_3$:

\begin{align*}
\begin{split}
& \max \left (L^{01}(h_n) - L^{01}(h'), L^{01}_n(h') - L^{01}_n(h_n)\right) \\
\leq & \frac{8}{n\phi} \left ( 4 C^2 V \log{n} + \left ( 1 + 2\phi \right) \log {\frac{1}{\delta_1}} \right ) 
 + 128\phi \left (d + C\sqrt{\frac{V \log(n) + \log{\frac{1}{\delta_2}}}{n}} \right ) \\
& + 128\phi C \max \left( w,\frac{1}{w} \right ) \sqrt{\frac{V \log(n) + \log{\frac{1}{\delta_3}}}{n}}
\end{split}
\end{align*}

\end{thm}
The key to interpreting this bound is to note that one tunes $\phi$
to optimize the value in $n$.  While the first term on the RHS 
dominates, let $\phi = 1$, but when this drops below subsequent terms,
$\phi$ can be set like $n^{-\nicefrac{1}{4}}$, giving an overall
convergence like $n^{-\nicefrac{3}{4}}$ down to a constant times $d$,
and then $\phi$ would be set like $n^{-\nicefrac{1}{2}}$ thereafter.  
Obviously, a small $d$ is important; 
that is to say, the privileged orderings should fit well
to some $h \in \Hs$ under $L^{iso}$.
However, $d$ could be much
smaller under favorable circumstances.
We believe it would be possible to characterize these
by extending the Mammen-Tsybakov noise conditions to 
$L^{iso}$.  A proof is given in the appendix.

\section{Optimizing $\mathbf{L_n^{iso}}$}

We now study methods to find $\inf_{h \in \Hs, m \in \Ms} L_n^{iso}(m \circ h)$,
assuming a linear or RKHS hypothesis space with 
defined level VC-dimension.  
We note that for a fixed $h$, the optimal $m$ can be found by means of a dynamic program.
A method to construct
real-valued functions with defined level VC-dimension
was given by Vapnik \cite[p.\ 359]{VapnikBigBook}.  The technique requires
the model to separate each successive example by a minimum
margin.
As common with zero-one loss, we relax it to hinge loss
to make a convex model.

Assume that examples are sorted increasing in $y_i$
and there are no duplicates (which require extra attention).
Define 
\[y_{ij} = \left \{ \begin{array}{ll}
-1 & y_j \leq y_i \\
1 & y_j > y_i
\end{array}
\right.
\]
In the following formulation, $C$ is a user-defined capacity control parameter and $\rho=1$
can be assumed:
\begin{align}
\min_{w,\xi,\zeta, l} & \frac{1}{2} \|w\|^2 + C \max_i(l_i) \\
\st & \mbox{for } i = 1\ldots n: \\
 & \; l_i = \sum_j \max \left (\rho-y_{ij}(w \cdot x_j - \xi_i), 0 \right ) \label{eqn:badLine}\\
& \mbox{for } i - 2\ldots n: \\
& \; \xi_{i-1} + \rho \leq \xi_i
\end{align}
This is a convex quadratic programming problem, though
regrettably requiring $n^2$ (where $n$ is the number
of examples) dummy variables to compute the 
max in line \ref{eqn:badLine}.  This makes
the program intractable (by machine learning standards),
at least without a special solver.

The regression problem can be made more tractable by
relaxing it to alternative ordinal regression formulations,
such as the one proposed by Sashua \& Levin \cite{SashuaLevin}.
Their formulation penalizes uses optimization variables
to define ordered slots according to the sort order
of the targets $y_i$, 
and a training example is 
penalized if it does not project into its slot.
It can be shown that the relaxation loss is an upper bound on the 
unrelaxed loss.
The relaxed formulation is a quadratic program with just $n$
constraints.

\section{GO-SVM}
The Global-Order SVM (GO-SVM) is the name for the formulation we propose.  It is simply
is the usual SVM
hypothesis space (thick hyperplanes) and loss (hinge),
but it is simultaneously optimized with the 
Sashua \& Levin \cite{SashuaLevin} ordinal regression relaxation,
with the SVM discriminant $w$ constrained
to be the same as the 
ordering hypothesis $w$.
This constraint implements 
the restriction from $\Hs$ (the SVM hypothesis space)
to $\hat \Hs$ (hypotheses that satisfy an ordinal condition), as defined 
in Theorem \ref{thm:riskBound}.
Loss and capacity control are traded off between the 
bi-objectives by means of user-selectable weights.

Capacity control in both SVM and the ordinal regression
formulations is attained by the relationship between the 
squared norm of the predictor $w$ and the size of the 
margin.  In either formulation by itself, one can fix
the margin size and place all capacity control in 
the squared norm of $w$, trading it off with loss. 
However, $w$ serves a two-fold
role in this formulation; therefore implementing different capacities 
for the two learning objectives
requires setting the margins.
Noting that the $\nu$-SVM formulations \cite{nuSVM} have the margin as an
optimization variable, we extended the approach so that
the usual tradeoff between loss and capacity is preserved.

The formulation is
\begin{align}
\begin{split}
\min_{\substack{\xi \geq 0\\
w,b,g,\xi\\\
\zeta, \rho_b, \rho_o}} & \frac{1}{2} w^T w  + \alpha \left ( -\nu_b \rho_b + \frac{1}{n} \sum_{i=1}^n \xi_i \right ) \\
 & + (1-\alpha) \left ( -\nu_o \rho_o + \frac{1}{n^*} \sum_{i=1}^n |\zeta_i| \right ) 
\end{split}\\
\st & \forall i,   y_i(w \cdot x_i +b) \geq \rho_b - \xi_i\\
& \forall i, g_{\mathcal{I}(i)} + \frac{\rho_o}{2} \leq w \cdot x_i + \zeta_i \leq g_{\mathcal{I}(i) + 1} - \frac{\rho_o}{2}
\end{align}
Here, $\mathcal{I}$ is an index function that returns an in-order, unique
index for each distinct oracle value for in each class,
and $g$ is a vector of interval boundaries.  Ordering
is enforced because there are no empty intervals.
The within-class ordering variant in principle
requires two ordinal regressions, but in practice
it can be done with a trick using the index function $\mathcal{I}$
by creating an empty interval.

Variable $w$ is the linear predictor, $b$ is a bias term, $\xi$ is the
hinge loss for the classification problem, and $|\zeta|$ is hinge loss
for the ordinal problem. Constant $n^*$ is defined to control the feasible
range of $\nu_o$.  It is $\nicefrac{n^2-n}{2}$ if there are not ties in the ordering.

Parameter $\nu_b$ controls
the VC-dimension of the 0/1 loss class,
$\nu_o$ controls the maximum VC-dimension of
each subproblem in $L^{thr}$.
Finally, parameter $\alpha$ is related to
choosing the size of $d$ in Theorem \ref{thm:riskBound},
expressed in terms of the permissiveness of the
ordinal loss compared to the 0/1 loss.
We do not attempt to enforce loss balance, as theory 
tells we should;
from any computed solution, we can still apply the bound for
as if we had constrained it to the value achieved by the optimum; 
moreover, we have never known the uncontrained optimum 
to have unreasonable loss balance, and
we have
no reason to prefer otherwise.

Like $\nu$-SVM \cite{nuSVM}, the optimization problem can be
characterized in terms of $\nu_b$ and $\nu_o$ and training data.
The
It can
be proved that the problem is primal and dual feasible
for $\nu_o \in [0,1]$, $\alpha \in [0,1]$, and $\nu_b \in [0, 2\min(\#\mbox{ positive examples}, \#\mbox{ negative examples})/n$; and primal unbounded/dual infeasible otherwise.
The Representer Theorem \cite{RepresenterTheorem} holds for GO-SVM, so the solution can be expressed in terms of the dual variables and kernels can be used.
We used Matlab's interior-point convex quadratic programming solver.
The baseline $\nu$-SVM formulation was also implemented using the same solver, so that differences in numerical accuracy could not arise.

\section{Evaluation}
The goal of evaluation
is to prove that the order oracle hypothesis space allows
faster convergence than a learning formulation which considers 
only the labels. Since GO-SVM is an extension of standard SVM, 
it is a logical baseline, and we compare only to that.
However, these experiments are similar to other common experiments
in LUPI literature, and we will point these to the reader where
appropriate.
Moreover, because of the construction of the
GO-SVM hypothesis spaces, it cannot outperform SVM by virtue of
a richer hypothesis space.  Faster convergence is the
only alternative explanation.
The evaluation is not intended to be a statement about
the fitness of the hypothesis spaces for the learning task,
but only about the ability of the learner to select the best element.

The experimental setup is to hold out a testing set and sample
remaining examples for 12 random realizations of training and validation sets.
The validation set is used
for a set of model selection experiments, and results are reported on the test set,
which is used for all experiments.  Testing sets contained at least 1800 examples.
The formulations have a fixed, auto-scaling parameter $\nu$, and we use
structural risk minimization to choose from a fixed set of
parameters $\nu = [.1:.1:.9, .95]$. 
  
The rbf kernel width (where used) is
chosen from the $[.1, .25, 5]$-quantiles of the pairwise distance
of training points.  The kernel parameter was chosen by 
a hold-out validation on the SVM experiment and re-used in the GO-SVM formulation to 
cut down the size of the model search.   The $\alpha$ parameter in the GO-SVM method was chosen from $[.1, .25, .5]$.

The first evaluation is up/down prediction of the
MacKey-Glass synthetic timeseries \cite{macKeyGlass}.
It was used in the
LUPI setting (SVM+) in \cite{VapnikLUPI}, where the authors
used a 4-dimensional embedding $(x_{t-3}, x_{t-2}, x_{t-1}, x_t)$
in order to predict $x_{t+5} > x_t$.
Privileged information was a 4-dimensional
embedding around the target: $(x_{t+3}, x_{t+4}, x_{t+6}, x_{t+7})$.
The authors compared SVM+ to SVM.  We were not able to
replicate their results for either SVM or SVM+, which
we suspect arises from the parameters used to generate the
timeseries.  (We used an integration step size of $.1$, with
points created every 10, delay constant $\tau=17$,
and initial value $.9$.)
We use $|x_{t+5} - x_t|$ as the order oracle.
We use an RBF kernel for all experiments with this dataset.

The second evaluation is predicting binary survival at a fixed time from onset.
We create synthetic datasets using the
same procedure as Shiao \& Cherkassky
\cite[personal communication]{hanTaiSurvival},
with noise level $.1$ and no censoring, which is given by
an exponential distribution parameter $\nicefrac{1}{.01}$.
While censored data are an inherent aspect of survival studies, 
we avoid it in this case because the ordinal model can be 
modified to accommodate the partial information that censored 
examples contain; thus, it is an experiment for another day.
They compare SVM, SVM+, and the Cox proportional hazards model.
Privileged information for SVM+ was related to the patient's overall
survival time and whether the event time is right censored
(only known to be greater than some value).
We use the (absolute) difference in the
fixed prediction horizon and the event time for the order oracle,
and we ignore whether an example is censored.  We consider only linear
models.

The last evaluation is handwritten digit recognition,
which was used by Vapnik and Vahist \cite{VapnikLUPI}
for SVM+ and slightly adapted
by Lapin \emph{et al.} for their proposed LUPI method \cite{reweightedSVMP}.
The task is to classify downsampled ($10 \times 10$) MNIST images
based on pixel values.  Lapin added
human-annotated confidence scores to training examples
(available for download).
We repeat the experiment using their data preparation
and using their annotators' confidence scores as the order oracle.
These experiments use an RBF kernel.

\subsection{Model selection}
The Go-SVM formulation considers a model space of 3 dimensions: a parameter
to control the complexity of the classification problem, a parameter to
control the complexity of the ordinal regression problem, and a 
parameter that balances the loss between these two problems.
All of the parameters are chosen from fixed lists, which 
were detailed supra.  The most basic form of model selection requires choosing
the best node of the grid.

We found that traditional hold-out
model selection strategies are more difficult
with GO-SVM.
The trouble appears to be because the
assumptions of structural risk minimization \cite{VapnikBigBook}
no longer hold.
In traditional SRM, the nested hypothesis spaces ensures that the 
loss expectation of the empirical risk minimizer 
(as a function of complexity) is coercive,
making the selection of a minimum considerably more reliable than 
if it occurred at random.
In our framework, there is no total ordering of hypothesis complexity.
Some hypothesis spaces (defined by parameters) 
have good convergence, while others do not.
The task is to differentiate them.

We analyzed loss surfaces with respect to various
dimensions of the parameter selection grid.
We used synthetic datasets where we could to generate a
great number of examples.  
We observed that, although the surfaces were not coercive,
they tended to be smooth.  Since the parameters of the
models have a direct interpretation, parameters that
are similar should have similar performance in expectation
when trained on the same set.  Thus, we decided to try a Gaussian
filter to smooth the validation results, and then
select the minimum in a grid search.  The filter
was constructed apriori, and was used for all the datasets in the evaluation.

In each experiment, we computed the loss on the test set that
would have been found by each of three methods:
\begin{enumerate}
\item A standard holdout, with the held-out validation set the same size as the
training set.  One might use cross-validation in practice.  This is called `unsmooth.'
\item The Gaussian smoothing technique using the same holdout.  This is called `smoothed.'
\item A very large `oracle' validation set which reveals how good the
best hypothesis space is.  This is called `extended.'
\end{enumerate}

The Gaussian filter was of size 5x5x3,
with the smallest dimension corresponding to the $\alpha$ parameter.
(In experiments using a kernel parameter,
we used one found via the SVM model search,
so this was not a grid search parameter.) 
It has the property that
any projection along coordinate directions of the filter is
Gaussian.
This was convolved with the
tensor of validation scores using
zero-degree smooth extrapolation; that is,
the tensor is padded out with constants
which are the same as the nearest true
element of the tensor.  The convolution
gives a new tensor that is the same
dimension as the un-smoothed tensor.

We also considered two alternative validation scenarios:
first, selecting a model based on the un-smoothed tensor, and second, investigating
the effect of having a much larger validation set available.  The large validation
set is intended to point out the gap between the best hypothesis spaces that can be
created using the ordinal constraint technique, and the one which can in practice be 
selected.  We point out that many LUPI research papers
require validation sets that would not ordinarily be a reasonable split between 
training and testing data (for example \cite{lapin2014learning, VapnikLUPI}.
Each row of the table gives the size of the training and test sets.  The columns
give the model selection procedure.

\begin{table}[tp]
\centering
\label{results}
\caption{Results (error rate) for all experiments.  Winners are reported excluding the extended validation experiments.}
\begin{tabular}{l c c  c c c}
& $\nu$-SVM & $\nu$-SVM & GO-SVM & GO-SVM & GO-SVM\\
experiment & std & ext & non-smooth & smoothed & extended \\
\toprule
MacKey-Glass (20, 4000) & .289 (.096) & .220 (.065) & \textbf{.156 (.105)} & .163 (.118) & .110 (.098) \\
MacKey-Glass (50, 4000) &  .117 (.040) & .096 (.032) & .058 (.024) & \textbf{.045 (.016)} & .035 (.014) \\
\midrule
Survival (20, 1000) & .409 (.030) & .399 (.206) & \textbf{.391 (.035)} & .397 (.030) & .357 (.030) \\
Survival (40, 1000) & .322 (.032) & .311 (.029) & .300 (.027) & \textbf{.287 (.024)} & .275 (.026) \\
Survival (100, 1000) & .243 (.020) & .235 (.014) & .221 (.023) & \textbf{.220 (.015)} & .207 (.009) \\
\midrule
Digits (60, 2500) & .113 (.022) & .108 (.020) & .109 (.020) & \textbf{.108 (.022)} & .102 (.021) \\
Digits (80, 2500) & .093 (.013) & .089 (.007) & .094 (.017) & \textbf{.089 (.008)} & .082 (.007) \\
\bottomrule
\end{tabular}
\end{table}

\subsection{Conclusions}
A table of results is given at Table (\ref{results}).  As a reminder, sizes
for training and testing are given in parenthesis with the experiment name.
The std, non-smooth, and smoothed methods used a validation set the same size as the training set.
We note first that the gap between the extended validation model selection 
and the performance of the typical technique is
larger for GO-SVM than for standard SVM.  This is a blessing in that we have the opportunity
to find a better model, but also a curse in that the variance is higher.
It appears that this strain of LUPI methods is bound by model selection issues.
The Gaussian smoothing approach seems to have been effective on the Digits dataset,
and certainly did not hinder performance significantly where the un-smoothed model
selection turned out to be superior.

The MacKey-Glass dataset is the only one which has strongly significant results.  
Although the gains in other datasets are small, the fact that they are supported by 
theory implies they should not be overlooked.  Moreover, they are consistent with results
reported by other authors.
There comparisons with other works for the MacKey-Glass and Digits experiments.
The MacKey-Glass experiment appeared in the original SVM+ paper \cite{VapnikLUPI}.
They report, based on a training set of 100 examples, that SVM had an error rate
of .052, whereas the best SVM+ formulation was at .048.  Furthermore, this
was based on a validation set of size 500.  We have reached that level of 
performance with considerably less data.  
The Digits experiment is intended to replicate one in \cite{lapin2014learning}.
We specifically replicated the experiment in which conditional probability
weights were created by humans with an intent to help a machine.  This
task is well-suited to the 
order invariance that GO-SVM is built on, as humans have a fundamentally ordinal notion
of confidence.  In that study at a sample size of 80, the difference between the best and worst methods under
study was about .01---.073 compared to .083 (approximately). The size of the validation set
in use, would be comparable to our extended experiment.  
Their best method, however,
did not use human weights.  In their experiment, the human weights 
information improved over SVM by about
.003, whereas our gain is .006.

In conclusion, the fact that the formulation can find faster-converging models than 
formulations which don't consider order information supports the underlying theory.
It appears likely the order information is helpful in scenarios when the prediction
task discretizes some continuous attribute, such as in the timeseries and survival
prediction tasks.

\section{Previous work}

The original SVM+ paper \cite{VapnikLUPI} touched off a
fair amount of research in the area.  Most research,
with limited exceptions, has focused on developing and evaluating
formulations \cite{adaBoostPP,
latentLUPI, metricLUPI, wang2015classifier, wang2015classifier}
rather than attempting to develop theory to
understand when and why such a technique might be useful.

Pechyony et.\ al.\ \cite{PechyonyTheory} analyze the SVM+ algorithm in terms of variance bounds.
While it shares with this work a major emphasis on variance bounds,
that work considers the SVM+ loss function as given and derives bounds for
it, whereas this paper works the other way in attempting to derive a formulation
based on the bound.

Lapin et.\ al.\ \cite{lapin2014learning} propose weighting examples based on class conditional
probability, and is most
directly similar to the ideas proposed here.
Intuitively, the method encourages a learner to prioritize
performance on the easy examples over the hard examples.
Unfortunately, the theoretical motivation for departing from
empirical risk minimization takes a tenuous path through SVM+ \cite{VapnikLUPI, PechyonyTheory},
namely that SVM+ is reducible to weighted learning.
The heart of their method is based on a loss function based
on weights interpreted as conditional probability; however,
a theoretical analysis is not provided.
Our is somewhat more general in allowing
order in variances.

\bibliographystyle{unsrt}
{\small 
\bibliography{svmp.bib}
}
\section{Appendix}
The first result is a bridge between the variance conditions
and the metric balls of hypotheses we can actually define.
This result shows that the uniform variance conditions can be
relaxed by a small constant, depicted here as $d(n)$, at the expense
of the rate of convergence when the bound is small.
\begin{thm}
Suppose there is a loss class $\mathcal{F} = \{l \circ h: h \in \Hs\}$
with VC-dimension $V$, and
let $f'$ attain $\inf_{f \in \Fs}\E[f]$.  There is a uniform constant $C$ such that
if the following holds uniformly for all $f \in \Fs$,
\begin{align*}
& \Var [f - f'] 
\leq \frac{1}{h} \E [f-f'] + d(n)
\end{align*}
then
for any $\phi \leq h \leq 1$, the following holds uniformly
with probability $1-\delta$:
\begin{align*}
& \max \left (\E[f_n - f'], \E_n[f' - f_n]\right) \\
\leq & \frac{8}{n\phi} \left ( 4 C^2 V \log{n} + \left ( 1 + 2\phi \right) \log {\frac{1}{\delta}} \right ) + 32\phi d(n).
\end{align*}
\end{thm}
\begin{proof}
I follow the definitions and notation in Bucheron \emph{et al.} \cite[Theorem 5.5]{BoucheronAdvances}. 
In their notation it is straightforward to show that $w(r) \leq \sqrt{\frac{r}{\phi} + d}$.
For VC-classes, it can be proved that $\psi(x) \leq C x \sqrt{\frac{V}{n} \log{n}}$ \cite{localRademacherComplexities}.
The risk bound depends on the solution of a fixed-point equation.  Let $\epsilon^*$ be the solution of
$r = \psi(w(r))$.
 Let  $\epsilon' = C^2 \frac{V \log{n}}{n\phi} + \phi d$.  The following analysis shows that $\epsilon' \geq \psi(w(\epsilon'))$, which implies $\epsilon^* \leq \epsilon'$.
\begin{align}
\psi(w(\epsilon')) 
&= C \left( \frac{C^2}{\phi^2} \frac{V \log{n}}{n} + 2d \right)^{\frac{1}{2}} \left ( \frac{V \log {n}}{n} \right )^{\frac{1}{2}}\\
&\leq C \left ( \left ( \frac{C^2}{\phi^2} \frac{V \log{n}}{n} \right)^{\frac{1}{2}} + \frac{1}{2} \left ( \frac{C^2}{\phi^2} \frac{V \log{n}}{n} \right )^{-\frac{1}{2}}2d \right) \left ( \frac{V \log {n}}{n} \right)^{\frac{1}{2}} \label{taylor}\\
&= C^2 \frac{V \log n}{n \phi} + \phi d = \epsilon'.
\end{align}
At step \ref{taylor} I used that the first-order approximation of the square root is an upper bound.
The bound $\epsilon'$ can be substituted wholesale into the bound given by Bucheron in the statement of the theorem, which gives the theorem.

\end{proof}

The next step is to show that the combination of the ordinal constraint and the balance constraint are sufficient to bound the variance diameter of the subset
of a hypothesis space that satisfies those conditions.

\begin{lem}
\label{pairwiseBound}
Suppose that an ordering and loss balance parameter $w$ are provided and 
and that $f,g \in \hat \Hs$.  That is,
$L^{iso}(f) \leq D_0$, $L^{iso}(g) \leq D_0$, and 
$L^B(f) \leq B_0$ and $L^B(g) \leq B_0$.
Finally, Suppose $L^{01}(g) \leq 
L^{01}(f)$.  Then $\E | ( l^{01}(f(X),Y) - l^{01}(g(X),Y) | \leq \E [l^{01}(f(X),Y) - l^{01}(g(X),Y)] + 4(b+d) $.
\end{lem}
\begin{proof}
For the purposes of the proof, we decompose the expectation by class.
Let $P_P = P(X|Y=1)$ and $P_N = P(X|Y=-1)$.  Let $p_P = P(Y=1)$ and
$p_N = P(Y=-1)$.  Similarly, let $L_P$ and $L_N$ be loss functions defined
by conditional expectations.
The outline of the argument is shown below.  We will expand each line subsequently.
\begin{align}
& \E_{P} | l^{01}(f(X),Y) - l^{01}(g(X),Y) |  \\
= & \E_{P_P} | l^{01}(f(X),1) - l^{01}(g(X),1) |p_P + \E_{P_N} | l^{01}(f(X),-1) - l^{01}(g(X),-1) |p_N\\
\leq & | \E_{P_P} l^{01}(f(X),1) - l^{01}(g(X),1) |p_P + | \E_{P_N} l^{01}(f(X),-1) - l^{01}(g(X),-1) |p_N + 4D_0 \label{varianceBound}\\
\leq & \E_{P_P}[ l^{01}(f(X),1) - l^{01}(g(X),1) ] p_P + \E_{P_N} [ l^{01}(f(X),-1) - l^{01}(g(X),-1) ]p_N + 4(D_0+B_0) \label{lossBalanceLine}\\
= \; & L^{01}(f) - L^{01}(g) + 4(D_0+B_0)
\end{align}

We begin by proving the inequality at line \ref{varianceBound}.
Consider only the positive class for a moment.
Let $\delta_f$ be the decision (or margin, as a simple extension) boundary for $f$ and $\delta_g$ the boundary for $g$.  Then $l^{01}$ is $1$ for $f(X) \leq \delta_f$ and $0$ otherwise.

We have noted the triangle inequality relationship between $D$ and $L^{iso}$. 
Suppose $L^{iso}(f) \leq D_0$
and $L^{iso}(g) \leq D_0$, then $D(f,g) \leq 2D_0$.  Let $m_f$ and $m_g$ be monotone functions
as defined (implicitly) in $L^{iso}$. Then $m_g^f = m_g^{-1} m_f $ is a continuous monotone function
which makes the metric relationship true.
We will first show
\begin{align}
\label{positiveVarianceBound}
& \E_{P_P} | \charfun_{f(X)\leq \delta_f } - \charfun_{g(X) \leq \delta_g} | \leq | \E_{P_P} [ \charfun_{f(X)\leq \delta_f } - \charfun_{g(X) \leq \delta_g}] | + 2D_0 \\
\Leftrightarrow & \E_{P_P} [ \charfun_{f(X) \leq \delta_f \; \wedge \; g(X) > \delta_g } ] + 
\E_{P_P} [\charfun_{f(X) > \delta_f \; \wedge \; g(X) \leq \delta_g} ] \\
& \leq | \E_{P_P} [ \charfun_{f(X) \leq \delta_f \; \wedge \; g(X) > \delta_g }] -
\E_{P_P} [\charfun_{f(X) > \delta_f \; \wedge \; g(X) \leq \delta_g} ] | + 2D_0
\end{align}
Rewriting concisely, we wish to show $a+b \leq |a-b| + 2D_0$.
In fact, this holds because
either $a \leq D_0$ or $b \leq D_0$,
which can be proved in the following way:  Expanding $a$ we have:
\begin{align}
\E [ \charfun_{f(X) \leq \delta_f \; \wedge \; g(X) > \delta_g }]
= & \E [\charfun_{f \leq \delta_f \; \wedge \; g > \delta_g \; \wedge \; g > m_g^f(\delta_f)}] + \E [ \charfun_{f \leq \delta_f \; \wedge \; g > \delta_g \; \wedge \; g \leq m_g^f(\delta_f)}]\\
\leq & D_0 + \E [ \charfun_{g > \delta_g \; \wedge \; g \leq m_g^f(\delta_f)}]
\label{zeroExpectationTerm}
\end{align}

The expectation term in line \ref{zeroExpectationTerm} is $0$ if $m_g^f(\delta_f) \leq \delta_g$.  Repeating the procedure for $b$ shows that the corresponding term is $0$ if $ m_g^f(\delta_f) \geq \delta_g$.  Since one of those conditions must be true, at least one of $a$ or $b$ is bounded by $d$, which proves line \ref{positiveVarianceBound}.  A bound for the negative class is identical. This proves inequality \ref{varianceBound}.

To prove inequality \ref{lossBalanceLine}, the
assumptions on class balance are needed.
Since we assumed that $L(f) > L(g)$, then either (or both)
$L_P(f) > L_P(g)$ or $L_N(f) > L_N(g)$.  If both are true,
then the desired inequality (\ref{lossBalanceLine}) is trivial.
Suppose that $L_P(f) > L_P(g)$ and $L_N(f) \leq L_N(g)$.
\begin{align}
& | L_P(f) - L_P(g) | p_P + |L_N(f) - L_N(g) |p_N \\
= &  |a -b | + |c-d| \\
= & a-b + d - c \\
= & a-b + c-d + 2(d-c)\\
\leq & a-b + c-d + 2w(b-a) + 4 B_0\\
\leq & (L_P(f) - L_P(g))p_P + (L_N(f) - L_N(g))p_N + 4 B_0
\end{align}
where we used that $b-a < 0$ ,$|wa-c| \leq B_0$, $|wb-d| \leq B_0$, and $w > 0$.
The proof if $L_N(f) > L_N(g)$ and $L_P(f) \leq L_P(g)$ uses
that $|a - \frac{1}{w}c| \leq B_0$ and $|b - \frac{1}{w}d| \leq B_0$.
\end{proof}

\end{document}